%% file: icml2026.tex
\documentclass{article}

\usepackage{microtype}
\usepackage{graphicx}
\usepackage{subfigure}
\usepackage{booktabs} 

\usepackage{hyperref}

\usepackage[preprint]{icml2026}

\usepackage{amsmath}
\usepackage{amssymb}
\usepackage{mathtools}
\usepackage{amsthm}

\usepackage[capitalize,noabbrev]{cleveref}

\input{preamble/packages}
\input{preamble/commands_general}
\input{preamble/commands_project} 

\newtheorem{theorem}{Theorem}[section]

\newtheorem{definition}[theorem]{Definition}
\newtheorem{assumption}[theorem]{Assumption}

\usepackage[textsize=tiny]{todonotes}

\icmltitlerunning{Position: Explainable AI is Causality in Disguise}

\begin{document}

\twocolumn[
\icmltitle{Position: Explainable AI is Causality in Disguise}



\icmlsetsymbol{equal}{*}

\begin{icmlauthorlist}
\icmlauthor{Amir-Hossein Karimi}{xxx}
\end{icmlauthorlist}

\icmlaffiliation{xxx}{Department of Electrical \& Computer Engineering, University of Waterloo \& Vector Institute, Canada}

\icmlcorrespondingauthor{Amir-Hossein Karimi}{amirh.karimi@uwaterloo.ca}

\icmlkeywords{Explainable AI, Causal Inference, Sufficiency \& Necessity}

\vskip 0.3in
]



\printAffiliationsAndNotice{}  

\begin{abstract}

\input{_sections/00_abstract}
\end{abstract}

\input{_sections/01_introduction}
\input{_sections/02_background_xai}
\input{_sections/03_background_scm}

\input{_sections/04_proof}

\clearpage
\input{_sections/05_future}
\input{_sections/06_conclusion}

\clearpage
\bibliography{refs}
\bibliographystyle{icml2026}

\onecolumn
\appendix

\clearpage
\input{_sections/091_sufficiency_necessity}

\clearpage
\input{_sections/092_more_causality}

\clearpage
\input{_sections/093_grad_questions}

\end{document}

%% file: preamble/packages.tex
\usepackage[T1]{fontenc}
\usepackage[utf8]{inputenc}
\usepackage{color, soul}
\usepackage{soul}
\usepackage{xspace}
\usepackage{xcolor}
\usepackage{adjustbox}
\usepackage{enumitem}
\usepackage{tikz}
\usetikzlibrary{positioning, shapes.geometric, arrows.meta}
\usepackage{multirow}
\usepackage{hhline}

%% file: preamble/commands_general.tex
\newcommand{\Mcal}{\mathcal{M}}

\ifx\BlackBox\undefined
\newcommand{\BlackBox}{\rule{1.5ex}{1.5ex}}  
\fi

\ifx\QED\undefined
\def\QED{~\rule[-1pt]{5pt}{5pt}\par\medskip}
\fi

\ifx\proof\undefined
\newenvironment{proof}{\par\noindent{\bf Proof\ }}{\hfill\BlackBox\\[2mm]}
\fi
\ifx\proofsketch\undefined
\newenvironment{proofsketch}{\par\noindent{\bf Proof Sketch\ }}{\hfill\BlackBox\\[2mm]}
\fi

\definecolor{rosso}{RGB}{220,57,18}
\definecolor{giallo}{RGB}{255,153,0}
\definecolor{blu}{RGB}{102,140,217}
\definecolor{verde}{RGB}{16,150,24}
\definecolor{viola}{RGB}{153,0,153}
\definecolor{babyblue}{RGB}{0,129,255}
\definecolor{darkgreen}{RGB}{6,148,60} 
\definecolor{darkblue}{rgb}{0.0,0.0,0.7} 

\providecommand*\url[1]{\href{#1}{#1}} 
\renewcommand*\url[1]{\href{#1}{\texttt{#1}}} 

%% file: preamble/commands_project.tex
\definecolor{whatcolor}{RGB}{153,204,255}
\definecolor{howcolor}{RGB}{255,204,153}
\definecolor{whycolor}{RGB}{204,255,153}
\definecolor{edgecolor}{RGB}{102,102,102}

\newcommand{\data}{\sethlcolor{whatcolor}\hl{data}\sethlcolor{yellow}\xspace}
\newcommand{\model}{\sethlcolor{howcolor}\hl{model}\sethlcolor{yellow}\xspace}
\newcommand{\decision}{\sethlcolor{whycolor}\hl{decision}\sethlcolor{yellow}\xspace}

\newcommand{\Data}{\sethlcolor{whatcolor}\hl{Data}\sethlcolor{yellow}\xspace}
\newcommand{\Model}{\sethlcolor{howcolor}\hl{Model}\sethlcolor{yellow}\xspace}
\newcommand{\Decision}{\sethlcolor{whycolor}\hl{Decision}\sethlcolor{yellow}\xspace}

\setlist[itemize]{left=3pt, topsep=2pt, partopsep=2pt, parsep=0pt, itemsep=2pt}
\setlist[enumerate]{left=3pt, topsep=2pt, partopsep=2pt, parsep=0pt, itemsep=2pt}

\newcommand{\doop}{\text{do}}

\usepackage{tcolorbox}

\tcbset{
    colback=gray!5!white, 
    colframe=white!45!black, 
    boxrule=0.5pt, 
    top=1.0mm, bottom=1.0mm, 
    left=1.5mm, right=1.5mm, 
    boxsep=1.0mm, 
}

\newcommand{\rev}[1]{\begingroup\color{black}#1\endgroup}

%% file: _sections/00_abstract.tex
The demand for Explainable AI (XAI) has triggered an explosion of methods, producing a landscape so fragmented that we now rely on surveys of surveys.
Yet, fundamental challenges persist: conflicting metrics, failed sanity checks, and unresolved debates over robustness and fairness.
\emph{The only consensus on how to achieve explainability is a lack of one.}
This has led many to point to the absence of a ground truth for defining ``the'' correct explanation as the main culprit.

\textbf{This position paper posits that} the persistent discord in XAI arises not from an absent ground truth but from a ground truth that exists, albeit as an elusive and challenging target: the \emph{causal model} that governs the relevant system.
By reframing XAI queries about \emph{data}, \emph{models}, or \emph{decisions} as causal inquiries, we prove the necessity and sufficiency of causal models for XAI.
We contend that without this causal grounding, XAI remains unmoored.
Ultimately, we encourage the community to converge around advanced \emph{concept} and \emph{causal discovery} to escape this entrenched uncertainty.

%% file: _sections/01_introduction.tex
\section{Introduction}

\input{_sections/fig_questions}

As early as the 1980s, the challenge of explainable AI (XAI) has been recognized as both critical and ambiguously defined~\citep{kodratoff1994comprehensibility}.
Numerous attempts to tackle this issue have led to a diverse array of methods, which are organized and categorized across various surveys.
Notable works include those focusing specifically on neural networks~\citep{yosinski2015understanding, montavon2018methods, samek2021explaining}, as well as broader surveys addressing explainable AI in general~\citep{doshi2017towards, dovsilovic2018explainable, hoffman2018metrics, guidotti2018survey, lipton2018mythos, adadi2018peeking, gilpin2018explaining, miller2019explanation, gunning2019xai, du2019techniques, tjoa2020survey, arrieta2020explainable, carvalho2019machine, murdoch2019definitions}.
With each of these survey papers exceeding 1,000 citations, it's perhaps enough to warrant a survey of surveys~\citep{speith2022review}.

Despite the very many attempts, the field continues to grapple with fundamental questions.
The definitions of \emph{explainability} and \emph{interpretability} may not always be agreed upon~\citep{preece2018stakeholders, ehsan2024social, leblanc2024seeking, namatevs2022interpretability, marcinkevivcs2020interpretability}, and debates over accuracy-explainability tradeoffs have split the community into proponents of \emph{inherent} vs. \emph{post-hoc} explainability approaches~\citep{rudin2019stop, gunning2019darpa, laugel2019dangers}. 
The lack of consensus over definitions and methodologies is further compounded by concerns over fairness~\citep{von2022fairness}, robustness~\citep{yeh2019fidelity, ghorbani2019interpretation, kindermans2019reliability, hamon2020robustness}, privacy violations~\citep{shavit2020extracting}, and the susceptibility of explanations to being manipulated or fooled~\citep{dombrowski2019explanations, shavit2019extracting, heo2019fooling, slack2020fooling, sullivan2022explanation, wickstromflexibility}.
Due to the lack of ground truth explanations, the community has been compelled to pursue an axiomatic framework for defining explainability~\citep{sundararajan2017axiomatic, janizek2021explaining, amgoud2022axiomatic}, yet, despite their axiomatic appeal, later work has shown failures in essential sanity checks~\citep{adebayo2018sanity, tomsett2020sanity, karimi2023relationship}.

%
\textbf{This position paper argues that XAI is fundamentally a supervised problem where the target is reality itself—in short, XAI is causality in disguise.}
While acknowledging the difficulty of obtaining this world model, we argue that the real barrier to consensus in XAI lies in the field’s near-total disregard for actively seeking it. Causal assumptions, we contend, are essential to bring coherence to XAI by addressing core questions through a principled lens; without such assumptions, XAI methods risk providing explanations that lack rigor, reliability, or generalizability.
Several studies have highlighted the importance of causality in XAI, identifying specific areas where a causal foundation could improve existing methods.
\citet{karimi2020probabilistic, karimi2021algorithmic} advocate for incorporating causal relationships into counterfactual explanations to enable actionable outcomes, while \citet{chou2021counterfactuals} and \citet{baron2023counterfactuals} critique existing counterfactual methods for lacking causal grounding, which they argue leads to spurious correlations and incomplete explanations.
Similarly, \citet{carloni2023causality} highlight the absence of causality in current XAI as a critical limitation, emphasizing its necessity for building trust in AI systems.
Finally, \citet{beckers2022causal} highlights causality's potential for action-guiding explanations in XAI, and \citet{chen2023causal} propose integrating causal discovery into XAI methods to enhance interpretability, leading to more actionable explanations.
%
%
%
We argue that grounding XAI in causal reasoning offers a principled lens to unify and clarify the field’s fragmented objectives.

%% file: _sections/fig_questions.tex
\begin{figure}[!t]
    \centering
    \resizebox{\linewidth}{!}{
        \begin{tikzpicture}[>=stealth, thick]

            \pgfmathsetmacro{\xCol}{2.6}         
            \def\boxW{3.8cm}                     
            \def\qW{9.2cm}                       
            
            \pgfmathsetmacro{\yStep}{-2.2}       
            \pgfmathsetmacro{\qOffset}{-1.0}     
            \pgfmathsetmacro{\titleY}{1.0}       
            
            \pgfmathsetmacro{\yLvlThree}{0}
            \pgfmathsetmacro{\yLvlTwo}{\yStep}
            \pgfmathsetmacro{\yLvlOne}{2*\yStep}
            
            \pgfmathsetmacro{\qYThree}{\yLvlThree + \qOffset}
            \pgfmathsetmacro{\qYTwo}{\yLvlTwo + \qOffset}
            \pgfmathsetmacro{\qYOne}{\yLvlOne + \qOffset}

            \tikzstyle{xainode} = [
                rectangle,
                draw,
                minimum width=\boxW,
                minimum height=1cm,
                align=center,
                font=\bfseries\small,
                rounded corners=4pt
            ]
            \tikzstyle{causalnode} = [
                rectangle,
                draw=black!60,
                densely dashed,
                fill=gray!10,
                minimum width=\boxW,
                minimum height=1cm,
                align=center,
                font=\small,
                rounded corners=4pt 
            ]
            \tikzstyle{qnode} = [align=center, text width=\qW, font=\footnotesize, text=black!90]
            \tikzstyle{coltitle} = [font=\bfseries\normalsize]

            \node[coltitle] at (-\xCol, \titleY) {Explainable AI};
            \node[coltitle] at (\xCol, \titleY) {Causality};

            \node[xainode, fill=whycolor] (decision) at (-\xCol, \yLvlThree) {Decision-Based\\(``Why (not)?'', $Y$)};
            \node[causalnode] (cf) at (\xCol, \yLvlThree) {\textbf{3. Counterfactual}\\$Y_{X=x'}(u)$};
            
            \draw[<->, line width=1.2mm, gray!40] (decision) -- (cf);
            
            \node[qnode] (q56) at (0, \qYThree) {
                \textbf{Q6:} Why would the \decision differ if the input had been different?\\
                \textbf{Q5:} Why does the model make a specific \decision for a given input?
            };

            \node[xainode, fill=howcolor] (model) at (-\xCol, \yLvlTwo) {Model-Based\\(``How?'', $f$)};
            \node[causalnode] (int) at (\xCol, \yLvlTwo) {\textbf{2. Interventional}\\$P(Y \mid do(X))$};
            
            \draw[<->, line width=1.2mm, gray!40] (model) -- (int);
            
            \node[qnode] (q34) at (0, \qYTwo) {
                \textbf{Q4:} How do the \model's internal mechanisms function?\\
                \textbf{Q3:} How does the \model transform inputs into outputs?
            };

            \node[xainode, fill=whatcolor] (data) at (-\xCol, \yLvlOne) {Data-Based\\(``What?'', $X$)};
            \node[causalnode] (obs) at (\xCol, \yLvlOne) {\textbf{1. Observational}\\$P(X,Y)$};
            
            \draw[<->, line width=1.2mm, gray!40] (data) -- (obs);
            
            \node[qnode] (q12) at (0, \qYOne) {
                \textbf{Q2:} What underlying factors generate the \data?\\
                \textbf{Q1:} What explains the distribution of the \data?
            };

        \end{tikzpicture}
    }
    \vspace{-4mm}
    \caption{
        %
        Core methods in XAI for explaining an ML model ($f : X \rightarrow Y$) are categorized by purpose into \data-based, \model-based, and \decision-based questions. By mapping these directly onto Pearl's Ladder of Causation, we reveal that solving XAI fundamentally requires answering causal inquiries.
    }
    \label{fig:questions}
    \vspace{-4mm}
\end{figure}
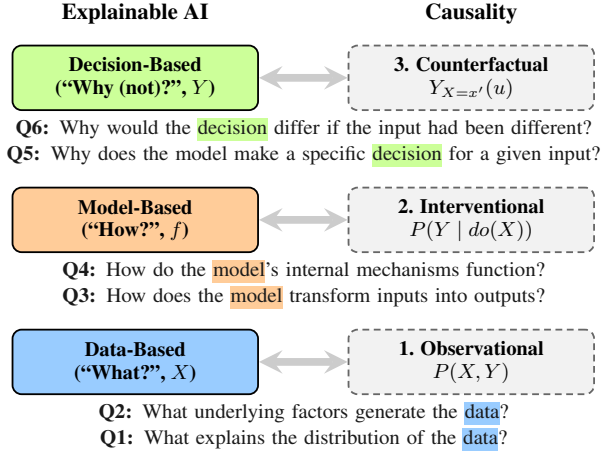

%% file: _sections/02_background_xai.tex
\section{Unifying XAI through Causal Semantics}
\label{sec:xai}

To understand the role of causality in explainable AI (XAI), we first categorize existing XAI methods based on the primary purpose of their explanations: the data $X$, the model $f$,\footnote{Here, $f$ represents the predictive model to be explained, distinct from the causal model of the world, $\Mcal$.} or the decisions $Y$.
As in \cref{fig:questions}, these questions can be organized into three categories of questions based on purpose:
\begin{itemize}
    \item \textbf{\Data-based (``What?'')}: Uncovering the structure and significance of the data $X$.
    \item \textbf{\Model-based (``How?'')}: Exploring how the model $f$ transforms input $X$ into output $Y$.
    \item \textbf{\Decision\footnote{Like \citet{miller2019explanation}, we use ``decision'' broadly as AI system outputs, such as categorizations or action choices.}-based (``Why (not)?'')}: Interpreting specific output $Y$ for given input $X$ and model $f$.
\end{itemize}
By structuring XAI methods within this framework, we highlight gaps due to a lack of causal grounding, setting a foundation for our argument that causality is essential for rigorous, valid XAI.

\subsection{\Data-Based Interpretability (``What?'')}

Data-based interpretability focuses on understanding the structure of the input data $X$, answering questions such as:
\begin{enumerate}[label=\emph{\textbf{Q\arabic*}}:, start=1]
    \item \emph{``What explains the distribution of the data?''}
    \item \emph{``What underlying factors generate the data?''}
\end{enumerate}

\noindent Data-based interpretability methods are particularly useful for exploratory data analysis and in contexts where understanding biases or clusters within the data is crucial, e.g.,

\begin{itemize}
    

    \item \textbf{Low-Dimensional Representations}
    methods such as Dimensionality Reduction (e.g., PCA, t-SNE~\citep{maaten2008visualizing}) and Manifold Learning (e.g., UMAP~\citep{mcinnes2018umap}, Isomap~\citep{tenenbaum2000global}) map high-dimensional data, $X$ to lower-dimensional spaces, revealing latent structure, clusters, or trends that explain the distribution of the data and impact predictions.

    \item \textbf{Clustering \& Group Discovery}
    methods such as k-means, DBSCAN~\cite{schubert2017dbscan}, and hierarchical clustering~\citep{murtagh2012algorithms} identify coherent subgroups within $X$, often reflecting hidden regimes, user subtypes, or covariate shifts, and their boundaries can influence fairness, recourse, or generalization.
    
    \item \textbf{Density \& Generative Modeling}
    techniques like Kernel Density Estimation~\citep{wkeglarczyk2018kernel}, Normalizing Flows~\citep{rezende2015variational}, or Variational Autoencoders (VAEs)~\citep{kingma2013auto} aim to estimate $p(X)$ or model the generative process behind the data. These probabilistic models capture the underlying distributional characteristics of $X$, enabling the detection of anomalies, rare events, and representative modes.
    
\end{itemize}


These method families align with the broader goal of understanding the intrinsic structure of data, and many share conceptual ties with \emph{causal discovery}—the identification of generative mechanisms and dependencies within $X$ itself~\citep{pearl2009causality, spirtes2001causation}.
Low-dimensional embeddings and manifold learning reveal smooth, interpretable representations of high-dimensional data; clustering highlights natural groupings; and density estimators and generative models characterize the full data distribution. 
Collectively, these tools expose latent patterns, confounders, and biases in the data that may shape model behavior.
They provide an essential foundation for explainability in AI systems, offering model-agnostic insights that complement model-based explanations like attention or attribution.

\subsection{\Model-Based Interpretability (``How?'')}

Model-based interpretability seeks to explain the function $f$, specifically how the model processes input $X$ to produce output $Y$. This category addresses questions such as:
\begin{enumerate}[label=\emph{\textbf{Q\arabic*}}:, start=3]
    \item \emph{``How does the model transform inputs into outputs?''}
    \item \emph{``How do the model's internal mechanisms function?''}
\end{enumerate}

\noindent Model-based interpretability is essential in regulatory and high-stakes environments where transparency into $f$’s workings is required.
These methods include:

\begin{itemize}

    \item \textbf{Feature Interaction}
    methods (e.g., Partial Dependence Plots~\citep{friedman2001greedy}, Accumulated Local Effects~\citep{apley2020visualizing}) explore interactions within $f$ by showing how different features affect $Y$. Partial Dependence Plots, for instance, illustrate the effect of one or two features on $Y$ while other features are kept constant, revealing interactions in $f$.
    
    \item \textbf{Feature Attribution}
    methods (e.g., LIME~\citep{ribeiro2016should}, SHAP~\citep{lundberg2017unified}) decompose $f(X)$ to assign an importance score to each feature in $X$, indicating its contribution to the output $Y$. Some works interpret these attributions as estimates of local (individual) causal effects~\citep{chattopadhyay2019neural}, suggesting that LIME can be approximated via input gradients in sufficiently smooth regions.

    \item \textbf{Saliency and Visualization}
    methods (e.g., Saliency Maps~\citep{simonyan2013deep}, Grad-CAM~\citep{selvaraju2016grad}) visualize gradients to identify important regions in $X$ that affect $Y$, such as which image pixels are influential in a prediction. Grad-CAM, for example, generates a heatmap highlighting image regions that impact the model’s output.
    

    \item \textbf{Surrogate and Simplified Models}
    aim to approximate complex models $f$ in specific regions using \emph{inherently interpretable} models (e.g., decision trees, linear models). \citet{towell1993extracting} extract rules to enhance interpretability in neural networks, and LIME provides local explanations through linear models~\citep{ribeiro2016should}. While MASALA adapts locality for improved fidelity~\citep{anwar2024masala}, MaLESCaMo introduces causal surrogate models~\citep{termine2023machine}, and \citet{laugel2018defining} focus on locality for surrogates in post-hoc interpretability.

    \item \textbf{Model-Intrinsic Interpretability}
    approaches use interpretable models like linear models, decision trees, and rule-based systems, allowing direct inspection of $f$'s parameters to understand how $X$ maps to $Y$ without post-hoc explanations. For instance, Generalized Additive Models (GAMs) model responses as sums of functions of predictors~\citep{hastie1987generalized}. The Bayesian Case Model uses representative cases for interpretability~\citep{kim2014bayesian}, while the Bayesian Rule Set framework learns interpretable rule sets~\citep{wang2017bayesian}. Interpretable Decision Sets provide a joint framework for description and prediction, facilitating comprehensible decision-making processes~\citep{lakkaraju2016interpretable}.

\end{itemize}

\rev{Mechanistic interpretability represents a prime example of successful, exact causal discovery on a model's computation graph. However, because tracking billions of low-level nodes does not scale to human comprehension, it highlights the urgent need for \textit{causal abstraction models} \cite{geiger2023causal} that group these low-level mechanisms into high-level, human-understandable causal variables.}

These approaches are closely related to understanding \textit{causal mechanisms}~\citep{pearl2000models, peters2017elements}—the specific processes through which changes in input features $X$ influence the output $Y$.
By attributing importance to features, analyzing interactions, and approximating internal model logic, these methods help uncover the pathways within $f$ that drive model predictions.
For example, feature attribution methods quantify each feature’s contribution to $Y$, aligning with causal mechanisms by revealing how particular inputs influence the model’s output.
Similarly, saliency maps and feature interaction methods highlight key regions and feature dependencies within $f$, providing an interpretative view of how the model operates.
This mechanistic understanding is essential in domains where stakeholders care to see not only which features matter but also how they interact to produce predictions.

\subsection{\Decision-Based Interpretability (``Why (not)?'')}

Decision-based interpretability focuses on explaining specific outputs $Y$ for given inputs $X$ and model $f$, addressing questions such as:
\begin{enumerate}[label=\emph{\textbf{Q\arabic*}}:, start=5]
    \item \emph{``Why does the model make a specific decision for a given input?''}
    \item \emph{``Why would the decision differ if the input had been different?''}
\end{enumerate}

\noindent Decision-based interpretability is valuable in applications where understanding the rationale behind individual decisions and possible alternatives is crucial, such as in personalized recommendations or legal judgments. Example methods include:

\begin{itemize}
    \item \textbf{Counterfactual and Example-based Methods}~\citep{wachter2017counterfactual} illustrate what minimal changes to $X$ would be necessary to alter the output $Y$, providing insight into decision boundaries by showing hypothetical scenarios in which the decision would differ.

    \item \textbf{Post-hoc Concept-based Explanation Methods} (e.g., TCAV)~\citep{kim2018interpretability} explain $Y$ in terms of high-level human-defined concepts, rather than individual features of $X$. TCAV, for example, assesses the relevance of specific concepts (like ``striped'' or ``curved'') to a prediction, offering an interpretable, concept-level explanation.

\end{itemize}

These methods draw on concepts from \textit{actual causality}~\citep{halpern2016actual} by using \textit{counterfactual reasoning} to explore why a particular outcome was reached. Halpern and Pearl's causal model formalizes this approach, defining causes through counterfactual dependencies that clarify necessary and sufficient conditions for an outcome~\citep{halpern2005causes}. In practical terms, answering ``why'' questions involves identifying the minimal changes in $X$ that would alter $Y$, thereby uncovering the causal factors influencing the decision~\citep{lipton1990contrastive}.
%
Counterfactual reasoning provides actionable insights, as it clarifies the conditions under which an alternative outcome could occur.
This concept of causality has also been extended by~\citet{woodward2005making}, who argues that interventions and counterfactuals provide a foundation for understanding causal explanations and model behavior. By leveraging such causal insights, decision-based interpretability approaches not only highlight decision boundaries but also enhance understanding of model outcomes and potential user actions.

\rev{To operationalize this taxonomy for natural language queries, one must identify the target of the inquiry. Questions regarding biases in the environment (e.g., ``Why are most successful applicants male?'') map to Q1/Q2 (the world's SCM). Questions regarding internal architecture or feature weights map to Q3/Q4 (the model's SCM). Questions regarding a specific user outcome (e.g., ``How can I change my loan rejection?'') map to Q5/Q6 (Decision).}

This purpose-driven categorization of \data-based, \model-based, and \decision-based XAI methods structures the response to XAI questions posed in~\cref{fig:questions}.
However, lacking causal assumptions limits robustness and generalizability across contexts.
Below, we introduce causal foundations and explore how causal models address these XAI gaps.
We will also see how existing lines of circuit-based interpretability and causal abstraction further strengthen the claim that \emph{explanation is causal discovery in disguise}.

%% file: _sections/03_background_scm.tex
\section{Background on Causality}

Causality aims to model the relationships between variables where one variable causes changes in another, thereby going beyond mere statistical correlations to capture the underlying mechanisms of the data-generating process. Unlike correlations, causal relationships entail directional influence, allowing one to predict the effect of interventions and counterfactuals in the system~\citep{pearl2009causality}.
Multiple frameworks formalize causality, including the Potential Outcomes framework~\citep{rubin2005causal}, Graphical Models~\citep{spirtes2001causation}, and Structural Causal Models (SCMs)~\citep{pearl2009causality}, each offering unique perspectives on understanding causation. For the purposes of this work, we adopt Pearl's SCM framework, as it provides a rigorous formalism for reasoning about causal mechanisms, interventions, and counterfactuals—critical components for constructing XAI systems.
We formalize the claim that access to the true causal model, represented as an SCM, is both \emph{sufficient} and \emph{necessary} for addressing purpose-driven methods on the ``What?'', ``How?'', and ``Why (not)?'' of explanations.
%
%
To ground our claims, we define key concepts and notations employed throughout this section (more in~\cref{app:more_causality}).

\begin{definition}[Structural Causal Model (SCM)]
    An SCM $\mathcal{M}$ is a tuple $\langle \mathbf{U}, \mathbf{V}, \mathbf{F}, P(\mathbf{U}) \rangle$, where:
    
    \begin{itemize}
        \item $\mathbf{U} = \{U_1, U_2, \dots, U_m\}$ is a set of exogenous variables.
        \item $\mathbf{V} = \{V_1, V_2, \dots, V_n\}$ is a set of endogenous variables.
        \item $\mathbf{F}$ is a set of structural equations ${ f_V : V \in \mathbf{V} }$, where each $f_V$ maps the parents of $V$ and relevant exogenous variables to $V$, i.e., $V = f_V(\operatorname{pa}(V), U_V)$. $\mathbf{F}$ specifies the causal mechanisms underlying the data-generating process, providing a mechanistic description of causal relationships.
        \item $P(\mathbf{U})$ is a joint probability distribution over the exogenous variables $\mathbf{U}$.
    \end{itemize}
\end{definition}

\rev{In the context of a deterministic neural network, the endogenous variables ($V$) are the layers and outputs, while the exogenous variables ($U$) represent the specific properties of the input instance $X=x$ that are fixed for that forward pass, alongside any external stochasticity (e.g., dropout masks).}

\begin{definition}[Causal Graph]
    The causal graph $\mathcal{G}$ associated with an SCM $\mathcal{M}$ is a directed acyclic graph (DAG) where nodes represent variables in $\mathbf{V}$, and edges represent direct causal relationships as specified by the structural equations in $\mathbf{F}$. $\mathcal{G}$ provides a visual representation of causal dependencies and is a fundamental tool for identifying causal pathways and potential confounders~\citep{pearl2009causality, spirtes2001causation}.
\end{definition}

\begin{definition}[Observ., Interv., and Counterf. Queries]
    Access to an SCM $\mathcal{M}$ enables analysis of three primary types of queries, each offering unique insights into the relationships captured by $\mathcal{M}$:
    
    \begin{itemize}
        \item \textbf{Observational Queries}: These involve probabilities computed from the observed data distribution $P(\mathbf{V})$. They describe associations between variables as observed without external manipulation and are limited to capturing correlations rather than causation.
        
        \item \textbf{Interventional Queries}: Interventions modify the underlying structural equations in $\mathbf{F}$ to estimate causal effects. Such interventions are denoted by the \textit{do}-operator, $\doop(\cdot)$, representing an exogenous alteration that severs the usual dependence of a variable on its causal parents, allowing for predictions under manipulated conditions. For example, the query $P(Y = y | \doop(X = x))$ estimates the probability of $Y = y$ when $X$ is set to $x$ by intervention~\citep{pearl2009causality}.
        
        \item \textbf{Counterfactual Queries}: Counterfactual queries explore hypothetical scenarios that diverge from observed reality, posing ``what if'' questions about alternative outcomes. For a given observed outcome, counterfactual reasoning considers what the outcome would have been had certain variables taken different values. This requires conditioning on observed data to infer observed values exogenous variables, $U=u$, and then modifying variables, $X=x'$, to then predict $Y_{X = x'}(u)$ counterfactuals~\citep{pearl2009causality, rubin2005causal}---a process otherwise known as the abduction–action–prediction paradigm. 
    \end{itemize}
\end{definition}

\begin{definition}[Causal Discovery]
    Unlike the queries above which presuppose a causal model, causal discovery~\citep{eberhardt2017introduction, glymour2019review, malinsky2018causal, nogueira2022methods, spirtes2016causal, vowels2022d} aims to infer the causal graph $\mathcal{G}$ from observational or experimental data, an essential step for constructing accurate causal models. This process faces challenges, including latent confounders, data scarcity, and reliance on assumptions like causal sufficiency. Methods for causal discovery include constraint-based approaches (e.g., PC algorithm)~\citep{spirtes2001causation}, score-based methods~\citep{huang2018generalized}, and functional causal models (e.g., additive noise models)~\citep{peters2017elements}. The ability to uncover causal relationships is crucial for XAI, as it directly affects the fidelity of the explanations generated.
\end{definition}

%% file: _sections/04_proof.tex
\section{Sufficiency and Necessity of Causality for Explainable AI} \label{sec:sufficiency-necessity}

\noindent In the following theorems, we first formalize the sufficiency claim, followed by the necessity claim.

\input{_sections/041_sufficiency_sketch}

\vspace{-7mm}
\input{_sections/042_necessity_sketch}

\vspace{-7mm}

\input{_sections/043_limitations}

%% file: _sections/041_sufficiency_sketch.tex

\begin{definition}[Accurate and Complete Answers to Q1-6]
Following \citet{pearl2009causality}, we say an answer to any of the six core XAI questions (Q1–Q6 in \cref{fig:questions}) is \emph{accurate and complete} if it coincides exactly with what the \textbf{true} Structural Causal Model (SCM) $\mathcal{M}$ predicts for that query. Concretely:
\begin{itemize}[leftmargin=*]
    \item \textbf{Observational correctness (Q1, Q2)}: The distribution of observed variables and the underlying generating factors match those in $\mathcal{M}$.
    \item \textbf{Interventional correctness (Q3, Q4)}: The effect of manipulating inputs or tracing internal mechanisms reflects the causal structure of $\mathcal{M}$.
    \item \textbf{Counterfactual correctness (Q5, Q6)}: The counterfactual outcome $Y_{X=x'}(u)$ for a specific exogenous state $u$ matches the counterfactual computed under $\mathcal{M}$.
\end{itemize}
\end{definition}

\begin{theorem}[Sufficiency of the True SCM for XAI]
\label{thm:sufficiency}
    Let $\mathcal{M} = \langle \mathbf{U}, \mathbf{V}, \mathbf{F}, P(\mathbf{U}) \rangle$ be the unique \emph{true} Structural Causal Model of the data-generating process. Under standard assumptions (acyclicity, no unmeasured confounders, well-defined exogenous variables), having full access to $\mathcal{M}$ is \textbf{sufficient} to provide accurate and complete answers to the six core XAI questions (Q1–Q6) depicted in \cref{fig:questions}.
\end{theorem}

\begin{proofsketch}
\textit{(Full proof in App.~\ref{app:sufficiency-necessity})} Since $\mathcal{M}$ specifies:
\begin{enumerate}[leftmargin=*]
    \item The causal graph $\mathcal{G}$ over the endogenous variables $\mathbf{V}$,
    \item A set of structural equations $\mathbf{F}$ indicating how each $V_i \in \mathbf{V}$ depends on its parents $\operatorname{pa}(V_i)$ and possibly exogenous $U_{V_i}$,
    \item The distribution $P(\mathbf{U})$ over the exogenous variables,
\end{enumerate}
it uniquely determines the joint distribution of all variables, any interventional distribution via the do-operator $\doop(\cdot)$, and any counterfactual query via abduction–action–prediction \citep{pearl2009causality}.
Mapping these distributions to to Q1–Q6:

\begin{itemize}[leftmargin=*]
    \item \textbf{Q1 (Distribution of data) \& Q2 (Underlying factors).} The law of structural models allows us to derive $P(\mathbf{V})$ exactly from $\mathcal{M}$, and we see how exogenous variables $\mathbf{U}$ and functions $f_{V_i}$ generate the observed data.

    \item \textbf{Q3 (How does the model process inputs?) \& Q4 (How do internal mechanisms operate?).} By tracing causal pathways in $\mathcal{G}$ (and applying $\mathbf{F}$ iteratively), we reveal how input $X$ propagates to output $Y$ through intermediate variables (hidden layers or sub-modules).

    \item \textbf{Q5 (Why a specific decision?) \& Q6 (Why would the decision differ?)}. 
        Given $(X = x, Y = y)$, we infer exogenous $\mathbf{u}$ (\emph{abduction}), modify $X \leftarrow x'$ (\emph{action}), and compute $Y_{X = x'}(\mathbf{u})$ (\emph{prediction}), explaining both why the model made its decision and how it would change under a different input.
\end{itemize}

\vspace{-2mm}

Because $\mathcal{M}$ yields precise observational, interventional, and counterfactual results, it provides complete and accurate explanations for all six questions. Thus, knowing the true SCM is \emph{sufficient} for XAI.
%
\end{proofsketch}

%% file: _sections/042_necessity_sketch.tex

\begin{assumption}[Q1--Q6 as a Separating Family]
\label{assum:separating_family}
    Let $\mathfrak{Q} = \{Q1, \dots, Q6\}$ denote the query functionals implied by Fig.~1. We assume $\mathfrak{Q}$ is separating for the model class under consideration. That is, for any two SCMs $\mathcal{M} \neq \hat{\mathcal{M}}$ in the class, there exists a query $Q \in \mathfrak{Q}$ such that $Q(\mathcal{M}) \neq Q(\hat{\mathcal{M}})$.
\end{assumption}

\begin{theorem}[Necessity of the True SCM for XAI]
\label{thm:necessity}
    Suppose a dataset $\mathbf{V}$ is generated by a \emph{true} but unknown SCM $\mathcal{M}$. If an alternative model $\hat{\mathcal{M}}$ does not match $\mathcal{M}$ in at least one structural equation or in its exogenous distribution $P(\mathbf{U})$, then there exists at least one of the six XAI questions (Q1–Q6) for which $\hat{\mathcal{M}}$ cannot provide an accurate and complete answer.
\end{theorem}

\begin{proofsketch}
\textit{(Full proof in App.~\ref{app:sufficiency-necessity})} Recall that accurate and complete answers require reproducing \emph{exactly} the observational, interventional, or counterfactual results from $\mathcal{M}$. We prove by contradiction:

\begin{enumerate}[leftmargin=*]
    \item \textbf{Assume} $\hat{\mathcal{M}}$ is a different SCM than $\mathcal{M}$ but still claims to yield correct answers for \emph{all} Q1–Q6.
    \item There are three broad query types:
        \begin{itemize}
            \item \textbf{Observational (Q1–Q2)}: If $\hat{\mathcal{M}}$ differs in $\mathbf{F}$ or $P(\mathbf{U})$, it may induce a different joint distribution over $\mathbf{V}$, 
            misidentifying underlying data factors (Q2).
            \item \textbf{Interventional (Q3–Q4)}: Even if $\hat{\mathcal{M}}$ matches observationally, the do-operator $\doop(X=x)$ can produce different outcomes in $\hat{\mathcal{M}}$ vs. $\mathcal{M}$ due to differences in causal structure or confounding assumptions \citep{pearl2009causality}.
            \item \textbf{Counterfactual (Q5–Q6)}: Counterfactual questions rely on abduction–action–prediction with the \emph{true} exogenous state. A mismatch in structural equations leads to different counterfactual results.
        \end{itemize}
    \item Hence, there must be at least one question Q1–Q6 where $\hat{\mathcal{M}}$’s answer diverges from $\mathcal{M}$’s. This contradicts the assumption that $\hat{\mathcal{M}}$ is correct for \emph{all} XAI questions.
\end{enumerate}

Thus, to ensure accuracy and completeness across all six questions, access to the true $\mathcal{M}$ is \emph{necessary} for XAI.
\end{proofsketch}

%% file: _sections/043_limitations.tex
\subsection{Discussion on Robustness and Limitations}

\paragraph{The Normative Ideal.}
\rev{We emphasize that Theorems \ref{thm:sufficiency} and \ref{thm:necessity} define a \textbf{normative ideal}. Defining a theoretical upper bound relative to an oracle SCM ($\mathcal{M} = \langle \mathbf{U}, \mathbf{V}, \mathbf{F}, P(\mathbf{U}) \rangle$) is not a tautology, but a necessary formalization to establish the target semantics of XAI. Just as the Bayes Optimal Classifier defines the theoretical limit of predictive modeling---an oracle we rarely access but continually approximate---the oracle SCM defines the theoretical limit of explainability. This frames all practical XAI methods as approximate causal discovery.}

\paragraph{Trading One Hard Problem for Another.}
However, in real-world applications, such oracle-level causal knowledge is rarely accessible. \rev{We recognize that by establishing this ideal, we are seemingly reducing the \emph{difficult problem of explainability} to the \emph{equally difficult problem of high-dimensional causal discovery}. However, we argue it is better for the community to tackle this hard problem explicitly than to continue developing yet another XAI method without a rigorous way to gauge its effectiveness. By making the causal ground truth explicit, we establish a principled metric for progress, discouraging the proliferation of heuristic methods that lack formal semantics.}

\paragraph{Challenges in Causal Discovery.}
Inferring the true $\mathcal{M}$ from data presents several well-known obstacles:
\begin{itemize}
    \item \textbf{Faithfulness and Causal Sufficiency.} Causal discovery typically assumes \emph{faithfulness} (i.e., observed independencies reflect true causal structure) and \emph{causal sufficiency} (i.e., no hidden common causes). If these assumptions fail, the inferred causal structure may be incorrect.\footnote{Consider three variables $(X, Y, Z)$ where $Z$ is an unmeasured confounder influencing both $X$ and $Y$ (i.e., $Z \rightarrow X$, $Z \rightarrow Y$). Without observing $Z$, the learned $\hat{\mathcal{M}}$ may wrongly suggest a direct causal link between $X$ and $Y$, leading to incorrect explanations.}
    
    \item \textbf{Sample Complexity and Computational Constraints.} Even when causal sufficiency holds, reliable causal discovery requires a large sample size, especially in high-dimensional settings. The number of samples required grows exponentially with the number of variables, making exhaustive search computationally infeasible~\citep{kalisch2007estimating}.

    \item \textbf{Identifiability and Equivalence Classes.} Even with unlimited data and valid assumptions, causal discovery methods often recover only a Markov equivalence class of DAGs—multiple causal graphs that imply the same observational dependencies~\citep{spirtes2001causation}. This ambiguity means that without interventional data, key causal relationships may remain unresolved.
\end{itemize}

\paragraph{Partial Causal Knowledge and Robustness.}
Because exhaustive causal discovery is often intractable, practitioners must rely on estimated, partial models ($\hat{\mathcal{M}}$) by incorporating known domain relationships. While this does not guarantee absolute correctness, it provides immense value over purely statistical methods. \rev{Specifically, the principle of Independent Causal Mechanisms (ICM) provides a powerful theoretical justification for causal XAI. When an environment undergoes a distribution shift, typically only one causal mechanism changes while others remain invariant. By anchoring explanations in causal abstractions rather than purely associational features, XAI methods inherit this invariance, providing explanations that remain robust across varying contexts.} To further assess robustness, sensitivity analysis~\citep{saltelli2004sensitivity} can quantify the stability of these explanations under small perturbations to $\hat{\mathcal{M}}$.

\noindent\textbf{Takeaway.}
The ideal of fully accurate and complete XAI is difficult to achieve due to the limitations above. In light of these challenges, correlation-based explanations (e.g., feature importances, saliency maps) may suffice when the goal is merely to detect patterns, biases, or anomalies rather than to enable interventions. Nonetheless, a more nuanced view is that the \emph{required level of causal grounding} depends on the stakeholder’s objective. When reliability matters—such as in high-stakes decision-making—approximate causal models, even if imperfect, yield explanations that are fundamentally more actionable and robust when addressing the diverse interpretability questions in Fig.~\ref{fig:questions}.

%% file: _sections/05_future.tex
\section{A Way Forward}

Recognizing these challenges, we propose strategic directions to address them, focusing on two interrelated tasks: \emph{Concept Discovery} and \emph{Relation Discovery}. By advancing methods in these areas, we can approximate causal models more effectively and enhance explainable AI. Despite the limitations, we encourage the community to embrace these challenges, as they are essential steps toward realizing that \emph{explainable AI is, in essence, causality in disguise}.

\subsection{Dual Challenges in Causal XAI: Concept Discovery and Relation Discovery}

\paragraph{Concept Discovery.} Effective explanations require a shared language of interpretable concepts $\{Z_i\}$ that align with the stakeholder's understanding. Explanations should be constructed using well-defined, semantically clear variables to ensure meaningful communication.
Current XAI methods vary along a \emph{Concept-Alignment Spectrum}:
\begin{itemize}
    \item \textbf{Fully Specified Concepts:} At one end, methods like SHAP~\citep{lundberg2017unified} and causal recourse~\citep{karimi2021algorithmic} provide explanations using features $X_i$ with direct semantic meaning, such as age or income. These methods produce mappings $\phi: X \rightarrow \mathbb{R}$ that quantify feature contributions and support actionable interventions.
    \item \textbf{Low-Level Features:} At the other end, methods like saliency maps~\citep{simonyan2013deep} highlight groups of pixels in images, which lack inherent semantic meaning and require abstraction to align with human concepts.
    \item \textbf{Concept-Based Methods:} In the middle, methods like TCAV~\citep{kim2018interpretability} attempt to align explanations with predefined concepts by measuring alignment with existing embeddings. However, TCAV is limited to known concepts and cannot discover new, relevant concepts---the ``unknown unknowns''---that may be crucial for understanding the model's behavior.
\end{itemize}

To enhance concept discovery, we advocate for methods that can uncover new concepts, potentially via causal approaches such as Concept Bottleneck Models~\citep{koh2020concept}, Causal Concept Effect~\citep{goyal2019explaining}, and Neuro-Symbolic Concept Learners~\citep{marconato2023interpretability, ellis2023dreamcoder} offering promising directions by treating concepts as entities that facilitate action and interpretability. These methods enable both structured learning and deeper understanding by integrating causal reasoning into concept discovery.
Concepts should also be identified at a granularity that is \emph{useful} to a given stakeholder (or audience).  Even if we had a perfect SCM of low-level features (e.g., pixels), explanations would remain unhelpful unless translated to higher-level abstractions that align with human mental models~\citep{beckers2019abstracting, rubenstein2017causal}.  Future research should thus emphasize learning and \emph{serving} these causal concepts at the right level of detail, possibly via user interaction or iterative refinement.

\paragraph{Relation Discovery.} Discovering causal relationships among identified variables $\{V_i\}$ remains a foundational challenge in interpretability. Traditional causal discovery and structure learning aim to infer a directed acyclic graph $\mathcal{G} = (\mathbf{V}, \mathbf{E})$, where $\mathbf{V}$ represents variables and $\mathbf{E}$ captures causal dependencies. Established algorithms like PC~\citep{spirtes2001causation} and score-based methods~\citep{huang2018generalized} provide structure but are often computationally demanding in high-dimensional settings.
We propose leveraging advances in \emph{causal representation learning}~\citep{bengio2019meta, scholkopf2021causal}, which strive to capture both the concept space and causal structure. These approaches can deepen interpretability by jointly learning representations that are both semantically meaningful and causally informative.
However, the scalability challenge becomes especially pressing for large-scale or high-dimensional models (e.g., LLMs).  Here, purely symbolic or conditional-independence-based causal discovery can be prohibitively slow. Exploring approximations such as sparse regressions, online structure learning, or domain-guided heuristics~\citep{granger1969investigating} may be necessary to handle real-world data at scale.

\subsection{Leveraging Approximate Models and Interactive Approaches}

In practice, obtaining a fully accurate causal model is often infeasible due to data and computational limitations. To address this, we advocate for \emph{approximate causal models} supplemented by interactive, user-driven methods. By iteratively refining causal structures through user feedback and interventions, approximate models can better align with real-world needs, enabling users to validate and adjust causal assumptions as needed.

In scenarios where full causal structure discovery is impractical, \emph{interactive approaches} enable iterative refinement of causal models based on user interactions and counterfactual queries. This user-in-the-loop methodology aligns with recent advances in chain-of-thought reasoning~\citep{wei2022chain} and large language models (e.g., GPT-4), allowing explanations to evolve with stakeholder feedback, enhancing their relevance and causal grounding.
Moreover, an interactive process can reveal the ``right'' level of abstraction for each user’s goals~\citep{teso2023leveraging}, acknowledging that an exhaustive model of the world is neither feasible nor desirable for most tasks. Instead, explanations should focus on \emph{those causal factors} that the user can understand and act upon, effectively capturing a subset of the world’s SCM aligned with the user’s mental model~\citep{gerstenberg2024counterfactual, gerstenberg2021counterfactual}.

\subsection{Summary of Recommendations}

In spite of limitations, our core thesis remains: \textit{XAI is causality in disguise}. Advances in concept and relation discovery will enable the construction of (approximate) causal models that enhance the rigor, reliability, and applicability of explanations.
We encourage the community to invest in:

\begin{enumerate}
    \item \textbf{Developing Robust Causal Discovery Algorithms:} Improving methods to better handle high-dimensional data, hidden confounders, and model misspecification. Future work should also explore multi-level abstractions~\citep{rubenstein2017causal,beckers2019abstracting} to balance expressivity with user interpretability.
    \item \textbf{Advancing Causal Representation Learning:} Jointly learning concepts and causal relations that are interpretable, stable, and scalable. Concept discovery should align with users’ internal models, recognizing that no single ``correct'' variable decomposition exists universally~\citep{teso2023leveraging}. This calls for methods that bridge machine-learned representations with human-understandable structures.
    \item \textbf{Promoting Interactive Explanations:} Engaging stakeholders in refining causal models through iterative feedback. 
    This aligns with ``explanatory interactive learning''~\citep{teso2019explanatory}, where users refine models by correcting explanations, steering causal learning to ensure relevance and actionability.
\end{enumerate}

By pursuing these directions, we can mitigate the practical challenges of causal XAI while moving toward a principled foundation for explainability. 

\subsection{Alternative Views: Possible Limitations of SCMs for Representing Human Intuition}

While we argue that explainable AI is fundamentally a causal problem, an opposing perspective questions whether Structural Causal Models (SCMs) are the right framework for capturing human reasoning used for explanation.
Specifically, SCMs are often criticized for their limited expressiveness in representing rich, structured mental models of the world.
Human reasoning frequently operates through \textit{intuitive theories}~\citep{gerstenberg2017intuitive}, which go beyond the propositional nature of SCMs.
For example, in physics, people intuitively understand the world in terms of objects, forces, and attributes (e.g., mass, elasticity, friction), rather than abstract causal graphs. When reasoning counterfactually, humans naturally ask questions such as ``What if this object hadn’t been there?'' or ``What if a reasonable person had acted differently?''—queries that are difficult to formalize in an SCM, where variables typically represent discrete events or predefined states.
Unlike SCMs, which encode causal mechanisms as structured equations over variables, human cognition often blends causal reasoning with spatial, temporal, and qualitative constraints, making it unclear whether SCMs are the best mathematical framework for modeling how people construct and interpret explanations.
%
For further discussion of SCMs and their comparison to Potential Outcomes, see Appendix~\ref{app:causality_sub_2}.

One response to this challenge is to extend SCMs with hierarchical abstractions that align with how humans structure knowledge. Recent work on \textit{causal abstraction models} and \textit{neuro-symbolic reasoning} offers promising directions by introducing layers of representation that move beyond traditional SCM constraints. However, these approaches remain an open area of research, and critics argue that a truly human-aligned XAI framework may require fundamentally different tools—potentially drawing from cognitive science, probabilistic programs, or physics-inspired models—to bridge the gap between mechanistic causality and intuitive human understanding.

%% file: _sections/06_conclusion.tex
\subsection{Conclusion}

The vast landscape of explainable AI is marked by an overwhelming number of methods, surveys, and perspectives, all of which underscore the field's current lack of consensus.
This paper argues that achieving such consensus hinges on viewing XAI through a causal lens, demonstrating through formal necessity and sufficiency results that \emph{causal assumptions} are both essential and adequate to address purpose-driven questions around the “What?”, “How?”, and “Why (not)?” of explanations.
By positioning explanations within a causal model, researchers and practitioners can align on clearer, more robust foundations for XAI, effectively viewing it as \emph{causality in disguise}.


Building on this viewpoint, we advocate for advancing \emph{concept discovery} and \emph{relation discovery} to identify variables and causal links at a level of abstraction that matches stakeholders’ mental models. In practice, approximate causal modeling and interactive refinement are key. By iteratively engaging users (e.g., through counterfactual queries or explanatory interactive learning), we can converge on a causal representation that offers actionable insights while accommodating the complexities of real-world systems.
Ultimately, we encourage the community to see beyond fragmented XAI methods and move toward a unified causal framework—one that embraces multi-level abstractions, interactive approaches, and real-world constraints. Although challenges like scalability, incomplete domain knowledge, and unmeasured confounders remain, they should be viewed not as barriers but as opportunities to refine and extend causal discovery methodologies for explainable AI. By doing so, we believe the field can progress toward a shared, actionable approach to XAI that balances rigor, utility, and adaptability for diverse stakeholders.

%% file: _sections/091_sufficiency_necessity.tex
\section{Sufficiency and Necessity of Causality for Explainable AI}
\label{app:sufficiency-necessity}
\input{_sections/041_sufficiency_proof}

\vspace{-3mm}
\input{_sections/042_necessity_proof}

%% file: _sections/041_sufficiency_proof.tex

\begingroup
\renewcommand{\thetheorem}{\ref{thm:sufficiency}} 
\begin{theorem}[Sufficiency of the True SCM for XAI]
    Let $\mathcal{M} = \langle \mathbf{U}, \mathbf{V}, \mathbf{F}, P(\mathbf{U}) \rangle$ be the unique \emph{true} Structural Causal Model of the data-generating process. Under standard assumptions (acyclicity, no unmeasured confounders, well-defined exogenous variables), having full access to $\mathcal{M}$ is \textbf{sufficient} to provide accurate and complete answers to the six core XAI questions (Q1–Q6) depicted in \cref{fig:questions}.
\end{theorem}
\addtocounter{theorem}{-1} 
\endgroup

\begin{proof}
    We proceed by examining each question individually, using the predefined variables and formal language.
    
    
    
    \begin{tcolorbox}[title=Q1: What explains the distribution of the data?]
        The SCM $\mathcal{M}$ specifies the structural equations $\mathbf{F}$ and the distribution $P(\mathbf{U})$. The joint distribution of the endogenous variables $\mathbf{V}$ can be derived from $\mathcal{M}$ using the \emph{law of structural models}:

        \vspace{-2mm}
    
        \[
            P(\mathbf{V}) = \int_{\mathbf{U}} \prod_{V_i \in \mathbf{V}} \delta \Big( V_i - f_{V_i}(\operatorname{pa}(V_i), U_{V_i}) \Big) P(\mathbf{U}) \, d\mathbf{U}
        \]

        \vspace{-2mm}
    
        where $\delta(\cdot)$ is the Dirac delta function ensuring that $V_i$ satisfies its structural equation, and $\operatorname{pa}(V_i)$ are the parents of $V_i$ in the causal graph $\mathcal{G}$ associated with $\mathcal{M}$.
        %
        Since we can derive $P(\mathbf{V})$ from $\mathcal{M}$, we can fully explain the distribution of the data, accounting for all dependencies and relationships specified by the structural equations and exogenous distributions.
    \end{tcolorbox}
    
    \begin{tcolorbox}[title=Q2: What underlying factors generate the data?]
        In the SCM $\mathcal{M}$, the exogenous variables $\mathbf{U}$ represent the underlying factors that are not determined within the model but affect the endogenous variables through the structural equations. Each endogenous variable $V_i$ is generated by: $V_i = f_{V_i}(\operatorname{pa}(V_i), U_{V_i})$.
        Access to $\mathcal{M}$ gives both the exogenous variables $\mathbf{U}$ and the structural equations $\mathbf{F}$, allowing us to identify and understand the underlying factors generating the observed data.
    \end{tcolorbox}
    
    \begin{tcolorbox}[title=Q3: How does the model~transform inputs into outputs?]
        Suppose the AI model takes inputs $\mathbf{X} \subseteq \mathbf{V}$ and produces outputs $\mathbf{Y} \subseteq \mathbf{V}$. The causal pathways from $\mathbf{X}$ to $\mathbf{Y}$ are specified in the causal graph $\mathcal{G}$ associated with $\mathcal{M}$.
        %
        The structural equations define how each variable depends on its parents: $V_i = f_{V_i}(\operatorname{pa}(V_i), U_{V_i})$.
        By following these equations along the paths from $\mathbf{X}$ to $\mathbf{Y}$, we can trace how inputs are transformed into outputs through the model. Specifically, we can compute the effect of $\mathbf{X}$ on $\mathbf{Y}$ by recursively evaluating the structural equations.
    
    \end{tcolorbox}
    
    \begin{tcolorbox}[title=Q4: How do the model's internal mechanisms function?]
        Internal mechanisms (e.g., hidden layers, intermediate computations) are represented by intermediate endogenous variables $\mathbf{H} \subseteq \mathbf{V}$ in the SCM.
        %
        The structural equations for the internal variables are:
        $
        H_j = f_{H_j}(\operatorname{pa}(H_j), U_{H_j})
        $
        By analyzing these equations and their dependencies, we can understand how the internal variables operate and contribute to the processing of inputs $\mathbf{X}$ to outputs $\mathbf{Y}$. The causal graph $\mathcal{G}$ of model $\mathcal{M}$ shows the connections between $\mathbf{X}$, $\mathbf{H}$, and $\mathbf{Y}$, allowing us to trace the flow of information and causation through the model's internal structure.
        %
        %
    \end{tcolorbox}
    
    \begin{tcolorbox}[title=Q5: Why does the model make a specific decision for a given input?]
        Given a specific input $\mathbf{X} = \mathbf{x}$ and the observed output $\mathbf{Y} = \mathbf{y}$, we can perform \emph{abduction} to infer the values of the exogenous variables $\mathbf{U} = \mathbf{u}$ consistent with these observations. Using the inferred $\mathbf{u}$ and the structural equations $\mathbf{F}$, we can then trace the causal pathways from $\mathbf{X} = \mathbf{x}$ to $\mathbf{Y} = \mathbf{y}$, identifying the causal mechanisms and intermediate variables that led to the decision.
    
    
    
    
    
    
    \end{tcolorbox}
    
    \begin{tcolorbox}[title=Q6: Why would the decision differ if the input had been different?]
        To answer this counterfactual question, we consider an alternative input $\mathbf{X} = \mathbf{x}'$ while keeping the inferred exogenous variables $\mathbf{U} = \mathbf{u}$ fixed at the values inferred during abduction. Finally, comparing the counterfactual output $\mathbf{Y}^{*}$ with the original output $\mathbf{Y} = \mathbf{y}$ to understand how and why the decision would differ under the alternative input.
    
    
    
    
    
    
    
    
    \end{tcolorbox}
    
    \paragraph{Overall Conclusion}

    In each case, access to the true SCM $\mathcal{M}$ provides sufficient information—whether through computing distributions, tracing causal pathways, or performing counterfactual reasoning—to accurately and completely answer each of the six XAI questions.
    
    
    

\end{proof}

%% file: _sections/042_necessity_proof.tex

\begingroup
\renewcommand{\thetheorem}{\ref{thm:necessity}} 
\begin{theorem}[Necessity of the True SCM for XAI]
    Suppose a dataset $\mathbf{V}$ is generated by a \emph{true} but unknown SCM $\mathcal{M}$. If an alternative model $\hat{\mathcal{M}}$ does not match $\mathcal{M}$ in at least one structural equation or in its exogenous distribution $P(\mathbf{U})$, then there exists at least one of the six XAI questions (Q1–Q6) for which $\hat{\mathcal{M}}$ cannot provide an accurate and complete answer.
\end{theorem}
\addtocounter{theorem}{-1} 
\endgroup

\begin{proof}
    We will demonstrate that without causal information—specifically, without access to the true SCM $\mathcal{M}$—it is impossible to answer the six core XAI questions. We proceed by addressing each question individually, using the predefined variables and formal language established earlier.

    \begin{tcolorbox}[title=Q1: What explains the distribution of the data?]
        Without causal information, we only have access to the observational distribution $P(\mathbf{V})$ of the endogenous variables $\mathbf{V}$. However, $P(\mathbf{V})$ encodes statistical associations but not causal relationships. Statistical dependencies in $P(\mathbf{V})$ can arise from various causal structures, such as direct causation, confounding, or even collider effects.
    
        \textbf{Illustrative Example}:
        Consider three variables $X$, $Y$, and $Z$ with the following causal structures:
        \begin{enumerate}
            \item \textbf{Confounding}: $Z$ is a common cause of $X$ and $Y$, i.e., $Z \rightarrow X$, $Z \rightarrow Y$.
            \item \textbf{Causal Chain}: $X$ causes $Z$, which in turn causes $Y$, i.e., $X \rightarrow Z \rightarrow Y$.
            \item \textbf{Collider}: $X$ and $Y$ both cause $Z$, i.e., $X \rightarrow Z \leftarrow Y$.
        \end{enumerate}
        All these structures can produce similar statistical associations between $X$ and $Y$ in $P(\mathbf{V})$. Without causal assumptions or knowledge of the underlying SCM, we cannot distinguish among these possibilities.
    \end{tcolorbox}

    \begin{tcolorbox}[title=Q2: What underlying factors generate the data?]
        In an SCM $\mathcal{M} = \langle \mathbf{U}, \mathbf{V}, \mathbf{F}, P(\mathbf{U}) \rangle$, the exogenous variables $\mathbf{U}$ and structural equations $\mathbf{F}$ define how the observed data $\mathbf{V}$ are generated:
        $
            V_i = f_{V_i}(\operatorname{pa}(V_i), U_{V_i}), ~ \forall ~ V_i \in \mathbf{V}
        $
        Without access to $\mathcal{M}$, we lack knowledge of both $\mathbf{U}$ (the unobserved factors) and $\mathbf{F}$ (the causal mechanisms). Consequently, we cannot accurately model the data-generating process.
    \end{tcolorbox}

    \begin{tcolorbox}[title=Q3: How does the model~transform inputs into outputs?]
        Suppose the AI model is represented as a function $f: \mathbf{X} \rightarrow \mathbf{Y}$. Without causal information, we can estimate the conditional distribution $P(\mathbf{Y} \mid \mathbf{X})$ from observational data. However, this distribution reflects statistical associations, not necessarily causal effects.
        Potential issues include:
        \begin{itemize}
            \item \textbf{Confounding}: A hidden variable $Z \in \mathbf{V}$ (or $Z \in \mathbf{U}$) affects both $\mathbf{X}$ and $\mathbf{Y}$, inducing spurious associations.
            \item \textbf{Reverse Causation}: The true causal direction might be $\mathbf{Y} \rightarrow \mathbf{X}$.
            \item \textbf{Feedback Loops}: Cyclic dependencies complicate the interpretation of $P(\mathbf{Y} \mid \mathbf{X})$.
        \end{itemize}
        
        Without the causal graph $\mathcal{G}$, we cannot compute the interventional distribution:
        $
            P(\mathbf{Y} \mid \doop(\mathbf{X} = \mathbf{x}))
        $
        which reflects the causal effect of setting $\mathbf{X}$ to $\mathbf{x}$.
    \end{tcolorbox}

    \begin{tcolorbox}[title=Q4: How do the model's internal mechanisms function?]
        Internal mechanisms involve the causal interactions among hidden or intermediate variables within the model.
        Let $\mathbf{H} \subseteq \mathbf{V}$ represent internal variables (e.g., hidden layers in a neural network). The structural equations for $\mathbf{H}$ and their causal relationships with $\mathbf{X}$ and $\mathbf{Y}$ are given by:
        $
            H_j = f_{H_j}(\operatorname{pa}(H_j), U_{H_j})
        $
        Without knowledge of $\mathcal{M}$, we cannot specify these equations or the causal graph $\mathcal{G}$, preventing us from understanding how $\mathbf{H}$ mediates between $\mathbf{X}$ and $\mathbf{Y}$.
    \end{tcolorbox}

    \begin{tcolorbox}[title=Q5: Why does the model make a specific decision for a given input?]
        Explaining a specific decision requires identifying the causal factors that led from the input $\mathbf{X} = \mathbf{x}$ to the output $\mathbf{Y} = \mathbf{y}$.
        To perform this explanation, we need to:
        \begin{enumerate}
            \item \textbf{Abduction}: Infer the exogenous variables $\mathbf{U} = \mathbf{u}$ consistent with $\mathbf{X} = \mathbf{x}$ and $\mathbf{Y} = \mathbf{y}$.
            \item \textbf{Trace Causal Pathways}: Use the structural equations to identify how changing $\mathbf{X}$ affects $\mathbf{Y}$.
        \end{enumerate}
        Without $\mathcal{M}$, we cannot perform abduction because $\mathbf{U}$ and $\mathbf{F}$ are unknown. Additionally, we cannot trace causal pathways without the causal graph $\mathcal{G}$.
    \end{tcolorbox}

    \begin{tcolorbox}[title=Q6: Why would the decision differ if the input had been different?]
        Answering this question requires \textbf{counterfactual reasoning}, which involves considering a hypothetical scenario where the input is $\mathbf{X} = \mathbf{x}'$ (different from the observed $\mathbf{X} = \mathbf{x}$) and determining the corresponding output $\mathbf{Y}_{\mathbf{X} = \mathbf{x}'}(\mathbf{u})$.
        As per \citet{pearl2009causality}, computing counterfactuals involves:
        \begin{enumerate}
            \item \textbf{Abduction}: Infer $\mathbf{U} = \mathbf{u}$ from the observed data $(\mathbf{X} = \mathbf{x}, \mathbf{Y} = \mathbf{y})$.
            \item \textbf{Action}: Modify the structural equations to reflect the counterfactual intervention $\doop(\mathbf{X} = \mathbf{x}')$.
            \item \textbf{Prediction}: Compute the counterfactual outcome $\mathbf{Y}_{\mathbf{X} = \mathbf{x}'}(\mathbf{u})$ using the modified model.
        \end{enumerate}
        Without the SCM $\mathcal{M}$, none of these steps can be performed accurately.
    \end{tcolorbox}

    \rev{\textit{Note on Markov Equivalence Classes (MECs):} If $\hat{\mathcal{M}}$ and $\mathcal{M}$ belong to the same MEC, they share the exact same observational distribution, meaning queries Q1 and Q2 will not separate them. However, models within an MEC must have at least one differing edge orientation. Consequently, an interventional query $do(X=x)$ on that specific edge, or a counterfactual query (Q5, Q6), will yield different distributions. Because our separating family $\mathfrak{Q}$ includes these interventional and counterfactual queries, Assumption~\ref{assum:separating_family} holds, successfully separating models even within the same MEC.}

    \paragraph{Overall Conclusion}

    The absence of causal information—specifically, the structural causal model $\mathcal{M}$—restricts us to the observational distribution $P(V)$, 
    preventing us from identifying underlying data-generating mechanisms, understanding causal pathways within the model, and performing counterfactual reasoning. Consequently, causal information is essential for providing accurate and reliable explanations in XAI. Without it, explanations may be incomplete, incorrect, or misleading.
\end{proof}

%% file: _sections/092_more_causality.tex
\section{A Guided Illustration of Causal Modeling Concepts}
\label{app:more_causality}

We expand on foundational causal modeling ideas through a concrete example as follows: 
\begin{enumerate}
    \item \cref{app:causality_sub_1} illustrates exogenous vs. endogenous variables,
    \item \cref{app:causality_sub_2} contrasts SCM and Potential Outcomes (PO) frameworks,
    \item \cref{app:causality_sub_3} highlights computational limits of exact discovery, and
    \item \cref{app:causality_sub_4} motivates open research questions (see ~\cref{app:grad_questions}).
\end{enumerate}

\subsection{SCM Basics via a Simple Example}
\label{app:causality_sub_1}
We consider a three-variable SCM:

\begin{align*}
    \texttt{SunriseTime}    &\rightarrow    \quad   \texttt{MoodToday} \quad \leftarrow \texttt{SleepHours} \\
    \texttt{MoodToday}      &=              \quad   f(\texttt{SunriseTime}, \texttt{SleepHours}, U)
\end{align*}

\texttt{SunriseTime} is an exogenous variable—fixed by environment or external conditions.
\texttt{MoodToday} and \texttt{SleepHours} are endogenous variables, defined through structural relationships.
Notably, both variable types may or may not be observed, and $U$ is unobserved noise.

\subsection{SCM vs. Potential Outcomes (PO)}
\label{app:causality_sub_2}

Our main position is not that structural causal models (SCMs) are the only viable framework for explainable AI (XAI), but rather that \emph{XAI questions are inherently causal}, and SCMs provide one principled framework to address them. Specifically, SCMs support observational, interventional, and counterfactual reasoning in a unified formalism that makes both conceptual assumptions and computational tradeoffs explicit~\citep{halpern2005causes, pearl2009causality}.

The Potential Outcomes (PO) framework, on the other hand, models causal effects as contrasts between unit-level responses under hypothetical interventions (e.g., $Y_i(1)$ and $Y_i(0)$)~\citep{rubin2005causal}. While PO is well-suited for estimating average treatment effects (ATEs) in randomized trials or observational studies, it is less expressive for answering individualized, mechanistic, or counterfactual XAI questions~\citep{imbens2020potential}.

For example:
\begin{itemize}
  \item PO-style query: What is the average effect of 8 vs. 4 hours of sleep on mood?
  \item SCM-style query: What would \emph{my} mood have been today, had I slept 8 hours instead of 4?
\end{itemize}

Moreover, SCMs naturally represent causal mechanisms, allow reasoning over latent confounders, and support interventions on learned internal representations—capabilities that are vital in modern machine learning systems.
That said, we acknowledge both frameworks can be reconciled under certain assumptions and offer complementary perspectives~\citep{imbens2020potential, pearl2019seven}. 

\subsection{Approximate vs. Exact Causal Discovery}
\label{app:causality_sub_3}

While SCMs offer a powerful lens for interpreting machine learning models, discovering the correct SCM from data remains fraught with theoretical and practical challenges:

\begin{itemize}
    
    \item \textbf{Non-identifiability.} Multiple SCMs can induce the same observational distribution. Without assumptions such as faithfulness, acyclicity, or no hidden confounding, it is impossible to infer the true causal structure solely from data~\citep{spirtes2001causation, pearl2009causality}.
    
    \item \textbf{Computational Hardness.} Even under identifiability, learning the DAG structure is NP-hard. Score-based approaches (e.g., GES) and constraint-based algorithms (e.g., PC, FCI) scale poorly as the number of variables increases. Recent work shows that many interpretability queries—such as counterfactuals or recourse—inherit this complexity~\citep{adolfi2024computational, barcelo2020model, bassan2024local}.
    
\end{itemize}

\textit{Example.} In our toy SCM:
\[
    \texttt{SunriseTime} \rightarrow \texttt{MoodToday} \leftarrow \texttt{SleepHours}
\]
if \texttt{SunriseTime} is latent, then even observing a statistical correlation between \texttt{SunriseTime} and \texttt{MoodToday} does not distinguish:
\begin{itemize}
  \item \texttt{SleepHours} $\rightarrow$ \texttt{MoodToday},
  \item \texttt{MoodToday} $\rightarrow$ \texttt{SleepHours}, or
  \item a hidden common cause.
\end{itemize}

Thus, causal discovery is often ill-posed and computationally expensive. As such, we advocate for \textit{approximate or partial} discovery: identifying abstractions or fragments of the true causal structure that are ``good enough'' to answer XAI queries of interest. This motivates research into $\epsilon$-approximate causal models and robustness analysis under model uncertainty---directions explored further in~\cref{app:grad_questions}.

\rev{

\subsection{LIME as Approximate Causal Inference}
We can view local surrogate models like LIME \cite{ribeiro2016should} as estimating a localized structural equation $Y \approx \sum w_i X_i$ around a specific input. However, because LIME relies on perturbing inputs independently, it ignores the true causal graph generating $X$. Consequently, it generates off-manifold, out-of-distribution samples. Under covariate shift, LIME's attributions often fail because the surrogate captures the model's arbitrary behavior on these off-manifold points rather than the true causal mechanism. Viewing LIME through an SCM lens immediately diagnoses this failure mode, demonstrating how making causal assumptions explicit provides a rigorous framework to gauge effectiveness, diagnose failures predictably, and ultimately improve community convergence on robust explanation standards.

\subsection{Algorithmic Recourse and Counterfactual Explanations}
Traditional counterfactual explanations seek to provide recourse by finding the minimal feature perturbation required to flip a model's prediction (e.g., ``If your education level were higher, your loan would be approved''). However, by treating features as independently manipulable, these methods ignore the underlying causal dependencies of the environment. In reality, features are causally intertwined; for instance, obtaining a higher education degree naturally affects both a person's age and future salary. If a stakeholder acts on a non-causal recommendation, the downstream effects on other variables are unaccounted for, frequently rendering the suggested action physically impossible or entirely ineffective at changing the actual outcome. As demonstrated in \textit{Algorithmic Recourse: from Counterfactual Explanations to Interventions} \cite{karimi2021algorithmic}, meaningful recourse must be formalized not as a geometric nearest-neighbor search, but as a causal intervention. Viewing counterfactual explanations through the lens of an SCM ensures that recommended actions respect the structural equations of the real world, transforming them into valid, actionable guidance.

}

\subsection{Where to Go Next}
\label{app:causality_sub_4}
The above conceptual challenges motivate practical and theoretical open research questions outlined in~\cref{app:grad_questions}.


%% file: _sections/093_grad_questions.tex
\section{Research Questions for Advancing Causal Explainability}
\label{app:grad_questions}

To help operationalize our position and inspire actionable follow-up, we present a set of concrete research questions suitable for graduate-level exploration. These questions span causal abstraction, approximate modeling, and stakeholder-aligned causal queries, and are supported by relevant literature to guide early-stage researchers.

\begin{table}[ht]
    \centering
    \renewcommand{\arraystretch}{2.4}
    \begin{tabular}{|c|c|p{5.75cm}|p{5.75cm}|p{3cm}|}
        \hline
            & & Question & Approach & References \\
        \hhline{=====}
            \multirow{2}{*}{\rotatebox[origin=c]{90}{\parbox[c]{5.5cm}{\centering Causal Abstraction \& \\Representation Learning}}}
            & RQ1
            & How can we systematically identify and extract human-interpretable concepts from deep neural networks (DNNs) to construct meaningful causal abstractions?
            & Develop unsupervised or semi-supervised algorithms leveraging neuro-symbolic methods to identify concept-level abstractions from latent spaces.
            & \citep{geiger2023causal, marconato2023not, zhang2024causal} \\ 
        \cline{2-5}
            & RQ2
            & Under what conditions can we guarantee that causal abstractions discovered from low-level neural representations remain faithful to the original neural network's causal structure?
            & Study theoretical guarantees for causal abstraction under approximate causal modeling frameworks. Quantify trade-offs between abstraction fidelity and interpretability.
            & \citep{beckers2019abstracting, karimi2023relationship, rubenstein2017causal} \\
        \hline
            \multirow{3}{*}{\rotatebox[origin=c]{90}{\parbox[c]{8cm}{\centering Approximate and Partial \\Causal Models}}}
            & RQ3
            & How can we formulate and efficiently compute approximate causal explanations that balance computational complexity with explanation accuracy?
            & Formalize a notion of $\epsilon$-approximate causal explanations and develop scalable algorithms to derive such explanations from large-scale models. 
            & \citep{glymour2019review, montagna2024score} \\
        \cline{2-5}
            & RQ4
            & Can we establish rigorous bounds on how approximations in causal discovery methods affect downstream XAI tasks such as model debugging or user trust?
            & Empirically and theoretically analyze sensitivity to approximation errors; develop robustness metrics for causal XAI methods.
            & \citep{janzing2020feature} \\
        \cline{2-5}
            & RQ5
            & What computational trade-offs arise when using approximate causal models in realistic XAI scenarios?
            & Study whether relaxing accuracy requirements yields tangible computational benefits or if computational complexity remains a core barrier.
            & \citep{adolfi2024computational, barcelo2020model, bassan2024local} \\
        \hline
            \multirow{2}{*}{\rotatebox[origin=c]{90}{\parbox[c]{5cm}{\centering Human-Aligned Causal \\Queries \& Interactive XAI}}}
            & RQ6
            & How should interactive XAI systems leverage causal interventions to enhance users' mental models about AI systems?
            & Design and validate interactive interfaces (e.g., visual tools or agents) enabling causal interventions, and assess user understanding.
            & \citep{miller2019explanation} \\
        \cline{2-5}
            & RQ7
            & Which classes of causal queries (observational, interventional, counterfactual) are most effective in various stakeholder contexts?
            & Conduct domain-specific studies to evaluate query impact on trust, decision-making, and transparency.
            & \citep{mittelstadt2019explaining, wachter2017counterfactual} \\
        \hline
    \end{tabular}
    \label{tab:research_questions}
\end{table}